\pdfoutput=1

\documentclass{article}
\usepackage{arxiv}
\usepackage{graphicx}
\usepackage{float}

\usepackage{amsmath}
\usepackage{amssymb}
\usepackage{mathtools}
\usepackage{amsthm}
\usepackage{bm}
\usepackage{dsfont}
\usepackage{times}
\usepackage[ruled,linesnumbered]{algorithm2e}

\usepackage{tabularx}
\usepackage{amsfonts}       
\usepackage{nicefrac}       
\usepackage{microtype}      
\usepackage{xcolor}         
\usepackage[colorlinks=true,citecolor=blue]{hyperref}       
\usepackage{enumitem}
\usepackage{subcaption}
\usepackage[english]{babel}
\usepackage[capitalize,noabbrev]{cleveref}
\usepackage{parskip}
\usepackage{overpic}
\usepackage{cite}
\usepackage{multirow}
\usepackage{makecell}
\usepackage{tablefootnote}

\theoremstyle{remark}

\theoremstyle{plain}
\newtheorem{theorem}{Theorem}
\newtheorem{proposition}{Proposition}
\newtheorem{corollary}{Corollary}
\newtheorem{lemma}{Lemma}

\theoremstyle{definition}

\newtheorem{assumption}{Assumption}

\crefname{equation}{Eq.}{Eqs.}
\crefname{assumption}{Assumption}{Assumptions}
\crefname{definition}{Definition}{Definitions}
\crefname{theorem}{Theorem}{Theorems}
\crefname{lemma}{Lemma}{Lemmas}
\crefname{proposition}{Proposition}{Propositions}
\crefname{remark}{Remark}{Remarks}
\crefname{corollary}{Corollary}{Corollaries}

\usepackage[toc,page,header]{appendix}
\usepackage{titletoc}

\usepackage{comment}

\usepackage{graphicx,accents}

\providecommand{\ind}[1]{\mathds{1}\cbrc{ #1 }}
\providecommand{\Ham}[1]{\mathrm{Ham}\brc{ #1 }}
\providecommand{\E}{\mathbb{E}}

\newcommand{\norm}[1]{\left\lVert #1 \right\rVert}
\newcommand{\abs}[1]{\left| #1 \right|}
\makeatletter
\DeclareRobustCommand{\rev}[1]{%
  \mathpalette\do@rev{#1}%
}
\newcommand{\do@rev}[2]{%
  \fix@rev{#1}{+}%
  \reflectbox{$\m@th#1\vec{\reflectbox{$\fix@rev{#1}{-}\m@th#1#2\fix@rev{#1}{+}$}}$}%
  \fix@rev{#1}{-}%
}
\newcommand{\fix@rev}[2]{%
  \ifx#1\displaystyle
    \mkern#23mu
  \else
    \ifx#1\textstyle
      \mkern#23mu
    \else
      \ifx#1\scriptstyle
        \mkern#22mu
      \else
        \mkern#22mu
      \fi
    \fi
  \fi
}
\makeatother

\providecommand{\KL}[2]{\mathrm{KL}( #1 || #2 )}
\providecommand{\TV}[2]{\mathrm{TV}( #1 , #2 )}
\newcommand{\brc}[1]{\left( #1 \right)}
\newcommand{\sbrc}[1]{\left[ #1 \right]}
\newcommand{\cbrc}[1]{\left\{ #1 \right\}}
\providecommand{\eps}{\varepsilon}

\renewcommand{\d}{\mathrm{d}}
\providecommand{\rmW}{\mathrm{W}}

\providecommand{\mbP}{\mathbb{P}}

\providecommand{\calN}{\mathcal{N}}
\providecommand{\calX}{\mathcal{X}}

\providecommand{\calO}{\mathcal{O}}

\newcommand{\poly}{\mathrm{poly}}

\providecommand{\mask}{\mathrm{[MASK]}}

\DeclareMathOperator*{\polylog}{polylog}

\providecommand{\yuchen}[1]{{\color{red} #1}}

\let\citep\cite

\title{From Scores to Gibbs Correctors: Accelerating Uniform-Rate Discrete Diffusion Models}

\author{Yuchen Liang \quad Ness Shroff \quad Yingbin Liang \\
The Ohio State University 
}

\begin{document}

\maketitle

\begin{abstract}
Discrete diffusion models have achieved strong empirical performance in text and other symbolic domains, but, especially for uniform-rate models, they often require many steps to generate a single sample. Existing acceleration methods either rely on training additional quantities or suffer from slow mixing. In this work, we propose a novel Gibbs-based corrector for discrete diffusion models, termed Gibbs-Accelerated Discrete Diffusion (GADD). GADD leverages the structure of the concrete score function to construct Gibbs posterior likelihoods directly, without requiring any additional training beyond standard score estimation. We show that GADD achieves an overall sampling complexity of $\mathcal{O}(\mathrm{polylog} (\varepsilon^{-1}))$, yielding the first such rate for diffusion-based samplers for uniform-rate discrete diffusion models. We also conduct numerical experiments demonstrating the practical advantages of GADD across synthetic data, zero-shot text sampling, and zero-shot conditional music generation. These results corroborate the theory and show that GADD consistently improves sample quality and wall-clock efficiency over standard baselines, including vanilla Euler methods and CTMC correctors. Beyond this, our theoretical analysis introduces a novel framework for analyzing predictor–corrector methods in discrete diffusion models, which may be of independent interest. Unlike existing approaches that rely on the Girsanov change-of-measure technique, our method is based on an induction argument that tracks error propagation across predictor iterations while accounting for inaccuracies in the corrector updates.
\end{abstract}

\section{Introduction}

Generative modeling is a central problem in deep learning, aiming to learn distributions that closely match observed data and enable high-quality sample generation. In recent years, diffusion models \cite{sohldickstein2015,ho2020ddpm,austin2021discrete} have emerged as a powerful and flexible generative framework, achieving state-of-the-art performance across a wide range of applications.
Among them, discrete diffusion models operate directly over discrete sample spaces and have proven highly effective for discrete-data domains. They have achieved strong empirical performance in natural language processing \cite{nie2025llada}, graph generation \cite{vignac2023digress,diffusion-survey-graph}, musical note generation \cite{campbell2022discrete}, and molecular and drug design \cite{diffusion-survey-drug-design,nisonoff2025dcfg}, among others.

Despite their empirical success, diffusion models face a major bottleneck in generation efficiency. Unlike classical generative models such as GANs, diffusion models typically require a large number of iterative steps to generate a single sample. A substantial body of work has explored accelerated sampling methods for {\em continuous} diffusion models, including predictor-corrector schemes \citep{song2020sde}, Runge-Kutta methods \citep{wu2024runge-kutta}, Hessian-based approaches \citep{liang2024discrete}, and, more recently, successive refinement techniques \citep{li2025higher-order}. In contrast, relatively few studies have focused on accelerating {\em discrete} diffusion models (see \Cref{sec:works} for details).

Among the two canonical classes of {\em discrete} diffusion models, {\em masked (i.e., absorbing-rate)} diffusion models (MDMs) exhibit unique structure induced by the masking dynamics: once a token reaches the absorbing masked state it remains fixed thereafter, which facilitates the development of faster samplers through structured unmasking steps. This built-in structure can be exploited to design efficient samplers. For example, \cite{zheng2025fhs} showed that MDMs admit an analytic characterization of sampling times, requiring at most $d$ steps to generate high-quality samples, and \cite{zhao2025informed} exploited informed correctors that are specially designed for MDMs.
In contrast, despite the strong language modeling performance of masked diffusion models (MDMs), {\em uniform-rate} discrete diffusion models remain widely used in domains such as graph and molecule generation, where masking-based approaches can suffer from state-clashing and produce invalid structures (e.g., molecules violating chemical rules) \citep{seo2025mask-graph,zhang2025mgdiff}. More recently, \cite{sahoo2026uniform-beat-mask} showed that uniform-rate models augmented with predictor-corrector methods can outperform both autoregressive (AR) and mask diffusion models in few-step text generation.

Despite promise, accelerating uniform-rate models is challenging and comparatively underexplored, since they lack the salient structure exploited by masking-based designs.
From a theoretical perspective, random-step samplers (e.g., uniformization) can simulate the reverse process exactly up to score-estimation error, but their realized step counts can be unbounded. For deterministic-step samplers, existing total-variation complexity bounds scale only polynomially in $\eps$, where $\eps$ is the target accuracy (see \Cref{tab:literature} for details): $\mathcal{O}(\eps^{-1})$ for exact and $\tau$-leaping methods \citep{campbell2022discrete,zhang2025conv-disc}, $\mathcal{O}(\eps^{-2})$ for Euler and Tweedie $\tau$-leaping \citep{liang2025sampler}. More recently, \citep{ren2025fast} proposed an accelerated approach achieving an improved rate of $\mathcal{O}(\eps^{-1})$ by leveraging higher-order numerical solvers to more accurately approximate the reverse process, inspired by advances in higher-order numerical solvers \citep{wu2024runge-kutta}. Overall, existing samplers for uniform-rate models achieve only {\em polynomial} dependence on $\eps$, which naturally raises the following question:

\emph{Question: Can we break the $\calO(\poly(\eps^{-1}))$ barrier and achieve a logarithmic dependence, namely $\calO(\polylog(\eps^{-1}))$, for uniform-rate discrete diffusion models?}




Our study provides affirmative answer to this question.

\subsection{Our Contributions}

In this paper, we propose the \textit{first} uniform-rate discrete-diffusion-based sampler that achieves $\calO(\polylog(\eps^{-1}))$ convergence rate. Our detailed contribution are as follows.

\textbf{Algorithmic Design:} Our key contribution is to propose a novel Gibbs-based corrector, the Gibbs-Accelerated Discrete Diffusion (GADD) algorithm, for sampling from uniform-rate discrete diffusion models. Notably, leveraging properties of the existing concrete score function, our GADD algorithm \emph{does not need extra-training} beyond the score functions. Instead, the posterior likelihood required for the Gibbs update can be directly obtained from the score estimator.

\textbf{Theoretical Analysis:} We theoretically show that GADD achieves an overall sampling complexity of $\mathcal{O}(\polylog (\eps^{-1}))$, improving upon existing diffusion-based sampling methods for uniform-rate discrete diffusion models. 
In comparison, all previous samplers only achieve the convergence rate of $\calO(\poly(\eps^{-1}))$.
In particular, we show that the diffusion process naturally provides a \emph{warm-start} to apply Gibbs correctors; such a warm-start simultaneously reduces both the Gibbs convergence error (under perfect per-step update) and the error from inaccurate score estimation.

\textbf{Numerical Experiments:} We further conduct numerical experiments demonstrating the practical advantages of GADD across synthetic data, zero-shot unconditional text sampling, and zero-shot conditional music generation. These results corroborate our theoretical findings and show that GADD achieves superior efficiency and robustness compared to previous approaches. In particular, GADD consistently improves sample quality at the same NFE budget, performs well on challenging spiky target distributions, and yields favorable wall-clock efficiency in realistic discrete-generation settings.


\textbf{General Framework for Predictor-Corrector Sampling:} Overall, we develop a general framework for analyzing predictor-corrector methods in discrete diffusion models, based on an induction argument that tracks error propagation across iterations while accounting for inaccurate corrector updates. By decomposing the error into initialization, mixing, and estimation components, the framework provides explicit conditions for controlling global error without stronger score accuracy requirements than in vanilla diffusion models. Beyond GADD, it also applies to the classical CTMC corrector \cite{campbell2022discrete}, offering a unified view of its convergence under approximate updates.

\begin{table*}[t]
    \centering
    \begin{tabular}{c|c|c}
        \textbf{Algorithm} & \textbf{Number of steps} & \textbf{Paper Reference} \\ \hline \hline

        Exact & $\widetilde{\calO}\brc{\frac{\sqrt{d}}{\eps}}$ & \cite{zhang2025conv-disc} \\ \cline{1-3}
        
        \multirow{3}{*}{$\tau$-leaping} & $\widetilde{\calO}\brc{\frac{d^4}{\eps}}$ & \cite{campbell2022discrete} \\ \cline{2-3}

        & $\widetilde{\calO}\brc{\frac{d^2}{\eps^2}}$ & \cite{ren2025stoc-int} \\ \cline{2-3}

        & $\widetilde{\calO}\brc{\frac{d}{\eps^2}}$ & \cite{dmitriev2026discrete} \\ \cline{1-3}

        Euler Method and Tweedie $\tau$-leaping & $\widetilde{\calO}\brc{\frac{d^2}{\eps^2}}$ & \cite{liang2025sampler} \\ \cline{1-3}
        
        DMPM (and its variants) & $\widetilde{\calO}\brc{\frac{d}{\eps^{4}}}$ & \makecell{\cite{conforti2025discrete} \\ \cite{conforti2025markov}} \\
        \hline \hline

        $\theta$-RK-2 and $\theta$-Trapezoidal & $\widetilde{\calO}\brc{\frac{\poly(d)}{\eps}}^\dagger$ & \cite{ren2025fast} \\ \cline{1-3}

        \color{blue} GADD (ours) & \color{blue} $\widetilde{\calO}\brc{\frac{\log^2(d/\eps^2)}{\rho_*}}$ & \color{blue} (This paper, Thm~\ref{thm:conv-unif})



        
    \end{tabular}
    \caption{{\em Summary of results for uniform-rate discrete diffusion samplers in terms of the number of steps needed to achieve $\calO(\eps)$ accuracy in $\TV{q_\delta}{\hat{p}_{T-\delta}}$ (or, equivalently, $\calO(\eps^2)$ for $\KL{q_\delta}{p_{T-\delta}}$), where  $q_\delta$ is a perturbed distribution satisfying $\TV{q_0}{q_\delta} \lesssim d \delta$. We only display those results having deterministic number of step-sizes.
    Here, $d$ denotes the dimension (e.g., length of the generated sentence), and $\rho_*$ is the path-wise worst-case spectral gap for Gibbs samplers (which is $\eps$-independent, and which is typically in the order $\Omega(d^{-1})$ for well-structured distributions). 
    $\dagger$ The $\poly(d)$ dependence is what we infer, which is not provided in their paper.
    Overall, our result achieves the first $\calO(\polylog (\eps^{-1}))$ convergence rate, whereas all prior works achieve polynomial dependence on $\eps^{-1}$.
    }}
    \label{tab:literature}
\end{table*}

\subsection{Related Works}
\label{sec:works}

In this subsection, we focus on related works that provide convergence guarantees for uniform-rate discrete diffusion models. Please see \Cref{app:works-more} for additional related works.

\textbf{Convergence Theory on Uniform-Rate Discrete Diffusion Samplers.} A growing body of work studies the number of steps required for samplers to achieve $\eps$-TV guarantees (see \Cref{tab:literature}). The earliest such result, \cite{campbell2022discrete}, analyzed $\tau$-leaping under the TV metric.
Subsequent studies considered two classes of samplers. For random-step methods, \cite{ren2025stoc-int,chen2024uniformization} established guarantees for the uniformization sampler, though the realized number of steps may be unbounded. For deterministic-step methods, \cite{zhang2025conv-disc} obtained strong guarantees under the assumption of access to a perfect per-step solver, while \cite{conforti2025discrete,conforti2025markov} achieved low dependence on $d$ via the DMPM sampler at the cost of a high $\eps$ dependence. \cite{ren2025stoc-int,dmitriev2026discrete} further improved guarantees for $\tau$-leaping, and more recently, \cite{liang2025sampler} provided the first analysis of the Euler method and Tweedie $\tau$-leaping. Notably, all these analyses of deterministic-step samplers yield $\calO(\poly(\eps^{-1}))$ convergence rates.

Concurrent with our work, \cite{kan2026adjoint} derives an $S$-independent upper bound on the error under the assumption of exact simulation of the continuous-time sampling process; however, the corresponding step complexity when there is discretization remains unclear.

\textbf{Acceleration of Uniform-Rate Discrete Diffusion Samplers.}
Compared to the extensive literature on standard samplers, relatively few works study the acceleration of uniform-rate discrete diffusion samplers.
Several empirical approaches have been proposed to accelerate discrete diffusion models. 
\cite{campbell2022discrete} first introduced a predictor-corrector scheme with a CTMC corrector that moves the particle toward the target distribution, and \cite{gat2024dfm} generalized this forward-backward idea and developed a unified corrector especially for discrete flow-matching. More recently, \cite{sahoo2026uniform-beat-mask} developed a family of $\Psi$-samplers that achieve strong performance in both language and image modeling. Notably, the correctors in \cite{sahoo2026uniform-beat-mask,gat2024dfm} are directly tied to the mean parameterization, in contrast to the score parameterization considered in our work.
On the theoretical side, \cite{ren2025fast} accelerated discrete diffusion models using higher-order numerical methods based on Runge-Kutta schemes, and achieves a convergence rate of $\calO(\eps^{-1})$ in terms of the target accuracy $\eps$.




\section{Preliminaries}
\label{sec:prelim}

In this section, we review both discrete diffusion models and Gibbs samplers, the two major components in the work.

\subsection{Discrete Diffusion Sampling}

Discrete diffusion models are defined by a forward noising process and a corresponding reverse denoising process involving \emph{discrete} data. The forward process is formulated as a continuous-time Markov chain (CTMC) on the discrete state space $\calX = [S]^d$, where $d$ denotes the number of tokens and each token takes values from a vocabulary of size $S$. Let $x_0 \in [S]^d$ denote the initial data sample, distributed according to the probability mass function (p.m.f.) $q_0$.
Let $R_t \in \mathbb{R}^{S^d \times S^d}$ be the (time-dependent) rate matrix of the CTMC. For any two states $x,y \in [S]^d$, the entry $R_t(x,y)$ specifies the instantaneous transition rate from $x$ to $y$ at time $t$. The conditional distribution of the state at time $t+\Delta t$ given the state at time $t$ satisfies
\begin{equation}\label{eq:def_forward}
    q_{t+\Delta t \mid t}(y \mid x)
    = \ind{y=x} + R_t(x,y)\,\Delta t + o(\Delta t),
\end{equation}
where $\ind{\cdot}$ denotes the indicator function. Validity of the CTMC requires that $R_t(x,y)\ge 0$ for all $x\neq y$ and $\sum_{y} R_t(x,y)=0$ for all $x$. Among the many choices  of $R_t$, one typical choice is the ``uniform-rate'' matrix \cite{lou2024entropy}. Specifically, the ``uniform-rate'' is defined as
\begin{equation*}
    R_t(x,y) = \begin{cases}
    \frac{1}{S} & \text{if}~\Ham{x,y}=1 \\
    0 & \text{if}~\Ham{x,y} \geq 2
\end{cases},
\end{equation*}
where $\Ham{x,y}$ denotes the Hamming distance between $x$ and $y$.
With such an $R_t$, the tokens would evolve independently and homogeneously for each token \citep{campbell2022discrete,lou2024entropy}. Also, one can show that $q_T \approx \mathrm{Uniform}([S]^d)$ and the mixing speed is exponentially fast \cite{zhang2025conv-disc}. 

To sample from the diffusion model, one typical approach is to walk through the reverse process, which is defined as the exact time reversal of the forward CTMC with initial distribution equal to $q_T$ \citep{campbell2022discrete,kelly2011reverse}. By \cite{campbell2022discrete}, the reverse process is itself a CTMC whose transition rate is $\rev{R}_{t}(x,y) := R_{t}(y,x) q_t(y)/q_t(x)$, which depends on the density ratio for all Hamming-distance-1 pairs $(x,y)$, a quantity commonly referred to as the \emph{(concrete) score}. Since this quantity is generally intractable, we estimate it using a neural network, with one popular choice of loss function: the score-entropy loss \cite{lou2024entropy}.
After we have obtained an approximate score, for practical sampling, we can discretize the continuous-time reverse process and employ approximate sampling methods for CTMC, such as the Euler method \cite{lou2024entropy}.
More specifically, define the estimated reverse rate as $\hat{R}_{t}(x,y) := R_{t}(y,x) s_{t}(y,x)$.
Upon initializing $x_{t_N} \sim \mathrm{Uniform}([S]^d)$ (with $t_N = T$), the Euler method is given by, for each $k=N-1,\dots,0$,
\begin{equation} \label{eq:def_euler}
x^i_{t_{k}} =
\begin{cases}
    a, & \text{w.p.}~\hat{R}_{k+1}^i(x_{t_{k+1}}^i,a) (t_{k+1}-t_k),~\forall a \neq x^i_{t_{k+1}}\\
    x^i_{t_{k+1}}, & \text{w.p.}~1 + \hat{R}_{k+1}^i(x_{t_{k+1}}^i,x_{t_{k+1}}^i) (t_{k+1}-t_k)
\end{cases},
\end{equation}
where
\[ \hat{R}_{k+1}^i(x^i,a) := \hat{R}_{t_{k+1}}(x, x^{-i} \oplus_i a),~\forall a \neq x^i, ~~\text{and}~ \hat{R}_{k+1}^i(x^i,x^i) := -\sum_{a \neq x^i} \hat{R}_{k+1}^i(x^i,a).\]

\subsection{Random-Scan Gibbs Samplers}

In the sampling literature, Gibbs sampling is among the most popular methods due to its simplicity, exactness of conditional updates, and ease of implementation. One typical Gibbs samplers is the (single-site) random-scan Gibbs sampler, where an update index is drawn randomly at each step. For a target distribution $\pi$ supported on $\calX$, define the one-step Markov transition kernel for random-scan Gibbs sampler as
\begin{equation} \label{eq:def-gibbs-kernel}
    P(x,y) := \begin{cases}
        w_i \pi^i(y^i|x^{-i}) & \text{if}~y^{-i} = x^{-i}\\
        0 & \text{otherwise}
    \end{cases}.
\end{equation}
Here $w_i$ is some weight function over the indices that sums up to 1.
For random-scan Gibbs samplers, such a kernel enjoys many regularity conditions of Markov operators. First, given its probabilistic nature, $P$ is aperiodic. Also, if $\pi(x) > 0$ over the support, $P$ is irreducible. Given these two properties, running such a Markov chain would yield a unique stationary distribution, which can be shown as $\pi$ itself. 
For a full review of Gibbs samplers (Glauber dynamics), we refer readers to \cite{levin2017markov-chain-book}. 

One way to characterize the convergence of such a Gibbs operator is through the Wasserstein-1 distance (associated with the Hamming metric), defined as:
\begin{equation*}
    \rmW_1(\nu_1, \nu_2) := \inf_{\Gamma\in\Pi(\nu_1,\nu_2)} \E_{(x,y) \sim \Gamma}\ \Ham{x,y},
\end{equation*}
where $\Gamma(\nu_1,\nu_2)$ is the set of couplings of $\nu_1$ and $\nu_2$.
Under certain classical assumptions (cf. Theorem 14.6 in \cite{levin2017markov-chain-book}), one can show that
\begin{equation} \label{eq:wass-conv-ass}
    \rmW_1(\nu_1 P, \nu_2 P) \leq (1-\rho) \rmW_1(\nu_1, \nu_2).
\end{equation}
Here $\rho$ is also called the \textit{spectral gap}, which depends on the structure of the underlying distribution.
The Wasserstein distance is closely related to the total-variation distance, which is a typical metric to characterize convergence in diffusion models (e.g., \cite{liang2025sampler,ren2025stoc-int}), as follows (see \cite[Proposition 4.7]{levin2017markov-chain-book}): 
\begin{equation} \label{eq:tv-wass-relationship}
    \TV{\nu_1}{\nu_2} \leq \rmW_1(\nu_1, \nu_2) \leq d \cdot \TV{\nu_1}{\nu_2}.
\end{equation}

\subsection{List of Notations}

Let $x^i (1\leq i \leq d)$ denote the $i$-th element of a vector $x \in [S]^d$ and $x^{-i} \in [S]^{d-1}$ denote the vector with the $i$-th element removed.
Define $\Ham{x,y}$ as the Hamming distance between two vectors $x$ and $y$.
For a positive integer $n$, $[n] := \{1,\dots,n\}$. 

\section{Discrete Diffusion Sampling with Gibbs Correctors}
\label{sec:GADD}

In this section, we present and analyze a novel Gibbs corrector for sampling from the uniform-rate CTMC, which we call the Gibbs-accelerated Discrete Diffusion (GADD) sampler.

\subsection{The GADD Algorithm}

We have summarized our algorithm in \Cref{alg:unif} below.

\begin{algorithm}
\caption{Gibbs-Accelerated Discrete Diffusion (GADD)}
\setcounter{AlgoLine}{1}
\SetKwInput{Input}{Input}
\SetKwInput{Return}{Return}

\Input{initialization $x_{t_N} \sim p_\text{init} = \mathrm{Uniform}([S]^d)$, discretization points $\cbrc{t_k}_{k=0}^N$ (with $t_N = T$ and $t_0 = \delta$), estimated score $s_t$, correction steps $\cbrc{L_k}_{k=0}^N$}
\label{alg:unif}
\For{$k = N-1$ \KwTo $0$}{
    $z_0 = $ the output of the Euler update given $x_{t_{k+1}}$ (using \eqref{eq:def_euler}) \;
    
    \For{$\ell = 1$ \KwTo $L_k$}{
        Sample $i$ w.r.t. weight $w_i$ over $[d]$\;
    
        Construct posterior $\hat{q}_{t_k}^i(\cdot|z_{\ell-1}^{-i})$ (e.g., using \eqref{eq:est-second} or \eqref{eq:est-first})\;
        
        Update $z_\ell^i \sim \hat{q}_{t_k}^i$ while fixing $z_\ell^{-i} = z_{\ell-1}^{-i}$\;
    }
    
    $x_{t_{k}} = z_{L_k}$\;
}
\Return{$x_{t_{0}}$}
\end{algorithm}

\Cref{alg:unif} comprises two components: an outer loop and an inner loop. The outer loop defines a sequence of target distributions ${q_{t_k}}$, each corresponding to a perturbed data distribution induced by the forward diffusion process. For each $t_k$, the inner loop runs a Gibbs sampler for $L_k$ steps targeting $q_{t_k}$. In practice, setting $L_k = 0$ corresponds to performing no correction at the current step. In particular, when $L_k = 0$ for all $k$ and the Euler update is applied throughout, \Cref{alg:unif} reduces to the standard Euler sampler defined in \eqref{eq:def_euler}.
We also note that the weights $w_i$ in the algorithm can be chosen in an informed manner, similar to \cite{zhao2025informed}. Finally, we provide a variant of \Cref{alg:unif} in \Cref{app:est-example}, which employs systematic-scan Gibbs updates rather than random-scan ones.




\subsection{Constructing Posterior Likelihood from Score}
\label{sec:pl-from-score}

To utilize the Gibbs corrector for target $q_t$ (where $t \in [\delta,T]$), one needs to have access to $q_t^i(\cdot|x^{-i})$ for all $t \in [\delta,T]$ and $i \in [d]$. Our key observation is that such information is readily available in the score function, which is the density ratio between all Hamming-distance 1 pairs. 
Notice that
\[ \textstyle q_{t}^i(x^i|x^{-i}) = \brc{ \sum_{y^i \in [S]} \frac{q_t^i(y^i | x^{-i})}{q_t^i(x^i | x^{-i})} }^{-1} = \brc{ \sum_{y^i \in [S]} \frac{q_t(x^{-i} \oplus_i y^i)}{q_t(x)} }^{-1}. \]
Thus, one example is to use the following estimator:
\begin{equation}\label{eq:est-second}
    \textstyle \hat{q}_{t}^i(x^i|x^{-i}) = \brc{\sum_{y^i \in [S]} s_t(x^{-i} \oplus_i y^i,x)}^{-1}.
\end{equation}
With a slight abuse of notation, we define $s_t(x,x) \equiv 1$ (which does not require any estimation). 
We have provided more estimator candidates in \Cref{app:est-example}.

To clarify, although the posterior estimator in \eqref{eq:est-second} (and also those in \Cref{app:est-example}) requires a summation over $[S]$, this does not necessarily require $S$ forward calls to the score model. For example, with the implementation of SEDD, a \emph{single} forward output of the score function is of size $(B,d,S)$ (where $B$ is the batch size), which contains all the corresponding scores (i.e., density ratios) used to aggregate. Thus, per Gibbs-step, we only need to call the score function \emph{once} to get an estimate of the posterior.

Such a construction enables direct evaluation of the posterior probabilities required by the Gibbs sampler, \emph{which does not require any extra-training and can be directly used for acceleration}. In contrast, previous correctors that employ Metropolis-Hastings algorithms \cite{huang2024rev-trans-kern} require an additional estimator for the likelihood ratio between all pairs of $(x,y)$, which is both computation- and memory-intensive given that $x,y \in [S]^d$ is high dimensional. With the estimator in \eqref{eq:est-second}, our Gibbs corrector avoids such auxiliary estimation and remains scalable to high-dimensional domains.

\subsection{Convergence Analysis of GADD} 

We first define a few useful notations. Let $p_{T-t_k,\ell}$ denote the sampling distribution at outer-loop step $k = N-1,\dots,0$ (with target $q_{t_k}$) and at inner-loop step $\ell = 1,\dots,L_k$, without estimation error. 
The following assumptions are needed for the score estimate.

\begin{assumption} \label{ass:score-ass}
Fix $t_k \in [\delta,T]$, $i \in [d]$ and $\ell \in [L_k]$. Suppose that score estimate satisfies that 
\[ \E_{\substack{x^{-i} \sim p^{-i}_{T-t_k,\ell} \\ x^i \sim q_{t_k}(\cdot|x^{-i}) }} \sum_{y:\Ham{x,y}=1} R_{t_k}(y,x) \abs{s_{t_k}(y,x) - \frac{q_{t_k}(y)}{q_{t_k}(x)}} \leq \eps_\text{est}. \]
\end{assumption}

Here we need a $L_1$-type score estimation error, which is different from the score-entropy error typically employed in the literature \cite{zhang2025conv-disc,ren2025stoc-int}. This estimation error arises because our analysis directly controls the total variation distance (see \Cref{prop:est-err}), a choice that has also been adopted in prior works \cite{campbell2022discrete,ren2025fast}. Additionally, for technical reasons, we require the expectation to be taken with respect to $p^{-i}_{T-t_k,\ell}$, which is a slightly perturbed version of the true target distribution $q_{t_k}^{-i}$. This condition is, however, likely to be satisfied when using a warmed Gibbs corrector, as is the case for GADD.

\begin{assumption}
\label{ass:score-reg}
Fix $t \in [\delta,T]$. Suppose that $s_t(y,x) \in [M^{-1}, M]$ for all $x$ and $y$.
\end{assumption}

The bound on $s_t$ is quite typical in prior literature \cite{zhang2025conv-disc,liang2025sampler}. Also, our main convergence result in \Cref{thm:conv-unif} only depends logarithmically on $M$.

With these assumptions, we are ready for the convergence guarantees of GADD.


\begin{theorem}[Convergence of GADD]
\label{thm:conv-unif}
Suppose that \Cref{ass:score-ass,ass:score-reg} are satisfied. Let $T \asymp \log(d (\log S)/\eps^2)$, 
$\eps_{\text{est}} = O \brc{ \frac{\rho_{\delta}}{M} \eps }$, $\delta \asymp \frac{\eps}{d}$, and $\kappa \asymp \brc{\log \log (d/\eps^2)}^{-1}$. Then, choosing $t_{k+1} - t_k \leq \kappa \min\cbrc{1, t_{k+1}}$, the total number of steps to achieve $\eps$ TV error for \Cref{alg:unif} suffices to have
\[ N_{total} \lesssim  \frac{\log(d/\eps^2) \log\brc{d^3 S/\eps^2 + d^2 M/\eps }}{\rho_*}, \]
where $\rho_* := \inf_{t \in [\delta,T]} \rho_t$. Note that we have omitted lower-order logarithmic factors.
\end{theorem}

\Cref{thm:conv-unif} provides the \emph{first} convergence guarantee for diffusion-based algorithms to achieve $\mathcal{O}(\polylog (\varepsilon^{-1}))$ convergence rate, which shows the effectiveness of predictor-corrector schemes with Gibbs correctors. In particular, for sampling from well-structured target distributions (e.g., high-temperature Ising models) where the spectral gap satisfies $\rho_* = \Omega(\poly^{-1}(d))$, the total number of sampling steps admits an explicit dependence on the dimension: $N_{\mathrm{total}} = \mathcal{O}\big(\poly(d)\polylog(\varepsilon^{-1})\big)$. This is a substantial improvement over prior acceleration methods with theoretical guarantees \cite{ren2025fast}.

On a high level, the diffusion process and the Gibbs sampler benefit from each other in GADD. First, the Gibbs corrector, through its cascading updates, achieves an exponentially decaying error at each outer loop, reducing the total number of required steps from $\calO(\poly(\eps^{-1}))$ to $\calO(\polylog(\eps^{-1}))$. On the other hand, the diffusion process naturally provides a \textbf{warm-start} condition that enables the Gibbs sampler to converge more effectively (see \Cref{lem:perstep-init}). This also mitigates estimation errors arising from an inaccurate corrector, as fewer corrector steps are required to reach convergence (see \Cref{prop:est-err,lem:est-err-diffusion}).

\subsection{Proof Sketch of \texorpdfstring{\Cref{thm:conv-unif}}{Theorem~2}}
\label{sec:thm:conv_tv_absorb_sketch}

Here we provide a proof sketch of \cref{thm:conv-unif} to describe the idea of our analysis approach.
The full proof is provided in \Cref{app:thm:conv-unif-proof}.
Overall, the main technical novelty is to involve inaccurate correctors in the analysis. Different from previous predictor-only approaches \cite{liang2024discrete,ren2025stoc-int,dmitriev2026discrete}, here the total convergence error cannot be easily decomposed via Girsanov's change-of-measure theorem. Rather, the backbone of our approach is an induction argument over the predictor steps.
Such a predictor-corrector analysis can also be extended to other correctors, including the CTMC corrector as we show in \Cref{subsec:campbell_corr} (see \Cref{thm:ctmc_corr}). 

Overall, the sources of error include the outer-loop and inner-loop initialization errors, as well as error propagation arising from imperfect score estimation during the inner-loop iterations. Among these, the outer-loop initialization error is straight-forward to derive, with $\TV{p_{\text{init}}}{q_T} \lesssim \sqrt{d \log S} \cdot e^{-T/2}$ (see \eqref{eq:grand-init-tv}, cf. \cite{zhang2025conv-disc,ren2025fast}) .
As follows, we focus on the inner-loop errors.



\textbf{Step 1: Bounding Inner-loop Initialization Error (\Cref{lem:perstep-init}).} 
Assuming that the previous outer-loop step ends with a small error, our goal is to show that this can be propagated to the start of the corrector steps, yielding a \emph{warm-up} condition. Indeed, such an error can be decomposed as (see \eqref{eq:lem:perstep-init-main})
\[ \TV{p_{T-t_{k},0}}{q_{t_k}} \leq \underbrace{\TV{\hat{p}_{T-t_{k+1},L_{k+1}}}{q_{t_{k+1}}} }_{\text{error inherited from previous step}}+  \underbrace{\TV{q_{t_{k+1}}}{q_{t_k}}}_{\text{error from moving target}} + \underbrace{\TV{\hat{p}_{T-t_{k+1},L_{k+1}}}{p_{T-t_k,0}}}_{\text{error from predictor step}}. \]
In the above, the first term is small by our assumption. As long as the step-size vanishes as $t_{k+1} - t_k = \kappa = o(1)$, the second and third terms are also small since the true and estimated scores have non-diverging upper-bounds. Here unlike vanilla discrete diffusion samplers, our choice of $\kappa$ does not need to scale as $\calO(\eps)$, which allows for significantly fewer outer-loop steps.


\textbf{Step 2: Bounding Inner-loop Estimation Error (\Cref{prop:est-err,lem:est-err-diffusion,lem:est-err-diffusion-other}).} 
We now turn to the second source of error, which arises from inaccurate correctors. This setting differs from standard analyses of Gibbs sampling, where the target posterior likelihood is typically known (see \Cref{app:works-more} for a survey). To study this effect, we first analyze error propagation in general random-scan Gibbs samplers with target distribution $\pi$ (\Cref{prop:est-err}); this result may be of independent interest. We then specialize to diffusion-path distributions, using the posterior likelihood estimators in \eqref{eq:est-second} or \eqref{eq:est-first} (\Cref{lem:est-err-diffusion,lem:est-err-diffusion-other}). In both cases, the error can be upper-bounded provided the estimated scores satisfy suitable lower-bound conditions. Notably, the required form of score estimation closely matches that used in vanilla diffusion models \cite{campbell2022discrete,liang2026sharp}.


\textbf{Step 3: Combining all parts together.} We now combine the previous ingredients to establish the induction framework. 
Suppose that at the previous step $t_{k+1}$, we have already shown
\[ \TV{\hat{p}_{T-t_{k+1},L_{k+1}}}{q_{t_{k+1}}} \lesssim L_{k+1} \eps, \]
where $L_{k+1}$ denotes the number of corrector steps in the preceding outer-loop iteration. Then, \Cref{lem:perstep-init} provides an upper bound on the initialization error for the next corrector loop.
We can further decompose the error accumulated within the corrector loop as
\[ \TV{\hat{p}_{T-t_k,L_k}}{q_{t_k}} \leq \underbrace{\TV{p_{T-t_k,L_k}}{q_{t_k}}}_{\text{error from pure Gibbs sampling}} + \underbrace{\TV{\hat{p}_{T-t_k,L_k}}{p_{T-t_k,L_k}}}_{\text{error from inaccurate Gibbs transitions}}. \]
Here the second term can be bounded using \Cref{lem:est-err-diffusion} as $\TV{\hat{p}_{T-t_k,L_k}}{p_{T-t_k,L_k}} \lesssim L_k M \eps_{\text{est}}$.
Since $\eps_{\text{est}} = O(\eps/M)$, the inductive hypothesis is preserved.
Meanwhile, by standard Gibbs sampling convergence results (cf. \eqref{eq:wass-conv-ass}), in order to achieve $\TV{p_{T-t_k,L_k}}{q_{t_k}} \leq \eps$, it suffices to take $L_k \lesssim \rho_{t_k}^{-1}\log\brc{\frac{d^2}{\eps} (M + S \delta^{-1}) }$.
Summing over $k$ yields the desired total number of steps. Notably, our analysis introduces no additional hidden factors beyond those arising from pure Gibbs sampling.

\section{Analysis of CTMC Corrector}
\label{subsec:campbell_corr}

In this section, we further apply the framework developed in \Cref{sec:thm:conv_tv_absorb_sketch} to other existing correction schemes for discrete diffusion models, thereby demonstrating the generality of our analysis.
A common class of correctors in the literature is based on CTMC dynamics. Specifically, at outer-loop time $t_k$, an additional CTMC with an appropriately chosen correction rate is simulated to target the distribution $q_{t_k}$. One example of such a correction rate is given by (cf. \cite[Section~4.4]{campbell2022discrete}):
\begin{equation} \label{eq:campbell-corr-def}
    R^c_{k}(x,y) := R_{t_k}(x,y) + \hat{R}_{t_k}(x,y).
\end{equation}
Indeed, \cite[Proposition 4]{campbell2022discrete} shows that the stationary distribution of $R^c_{k}$ is correct. As follows, we analyze the non-asymptotic performance of such a corrector using the framework of \Cref{thm:conv-unif}.

\begin{theorem}
\label{thm:ctmc_corr}
Suppose perfect score estimation and that \Cref{ass:score-reg} holds. In order to achieve $\eps$-TV error, the total number of steps needs to satisfy (ignoring log-dependencies)
\[ N_{total} \lesssim \frac{d^2 M^2}{\mu_*^2 \eps^2}, \]
where $\mu_* := \inf_{t_k \in [\delta, T]} \mu_t$ is the minimum modified log-Sobolev constant of the rate matrix $R^c_{k}$.
\end{theorem}

The proof of \Cref{thm:ctmc_corr} is provided in \Cref{app:thm:ctmc_corr-proof}. This result readily extends to the broader class of all CTMC correctors with the correct stationary distribution. In particular, \Cref{thm:ctmc_corr} shows that, \emph{even in the absence of estimation error}, achieving $\eps$-accuracy in TV requires $\calO(\poly(\eps^{-1}))$ total steps.
Compared to \Cref{thm:conv-unif}, this highlights the advantage of our proposed GADD sampler, which employs Gibbs correctors instead of CTMC-based ones. The proof of \Cref{thm:ctmc_corr} closely parallels that of \Cref{thm:conv-unif}, relying on the same induction strategy to handle the evolving target distributions and to control the accumulation of corrector errors due to a finite number of corrector steps.

Our analysis also sheds light on why CTMC correctors are inferior to Gibbs-based ones. Although the continuous-time CTMC converges to the target at an exponential rate, practical implementations require \emph{discretizing the continuous-time trajectory} (as in \cite{campbell2022discrete,gat2024dfm}). This discretization introduces a non-negligible error that significantly slows convergence, leading to an overall complexity of only $\calO(\poly(\eps^{-1}))$ (see \Cref{lem:campbell-discretization-err}).

\section{Numerical Experiments for GADD}
\label{sec:experiment}


In this section, we evaluate the numerical performance of our GADD algorithm on three datasets. 
We include all experiment details in \Cref{app:details}.

\subsection{Synthetic Experiment}

We first numerically compare our GADD algorithm in \Cref{alg:unif} on synthetic data against three baselines: (i) the Euler method, (ii) the Gibbs sampler, both of which are widely used and strong vanilla baselines for sampling from discrete data, and (iii) the $\theta$-Trapezoidal method \cite{ren2025fast}, a recently proposed accelerated sampler for uniform-rate discrete diffusion models.

\begin{figure}[ht]
\vspace{-3mm}
\begin{minipage}{0.45\linewidth}
\centering
\vspace{-5mm}
\includegraphics[height=6cm]{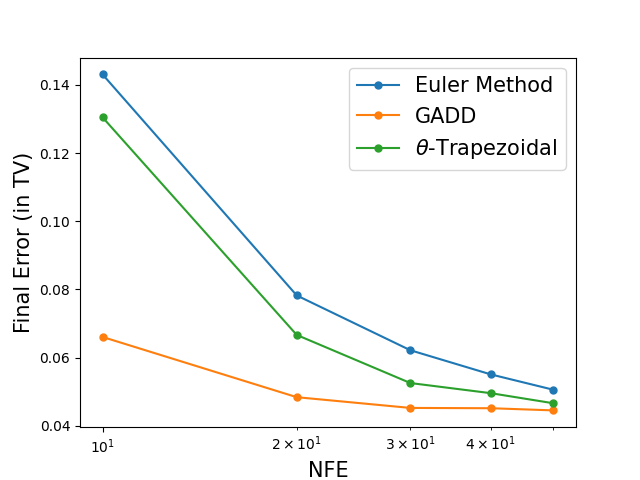}
\end{minipage}
\hfill
\begin{minipage}{0.45\linewidth}
\centering
\includegraphics[height=6cm]{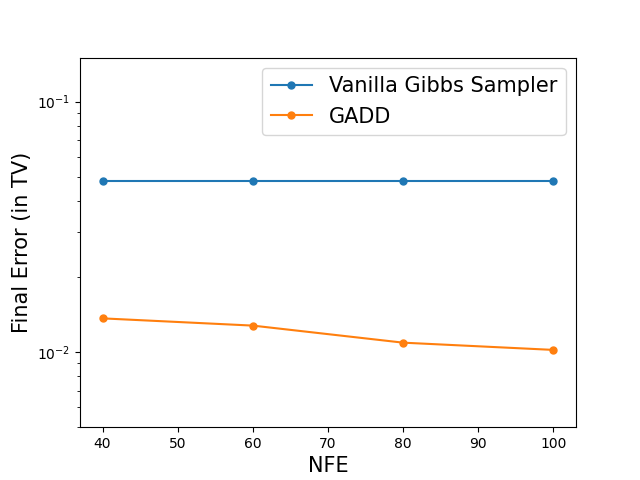}
\label{fig:gibbs}
\end{minipage}
\caption{
Comparison of GADD on synthetic data against the Euler method, the $\theta$-Trapezoidal algorithm (left), and the Gibbs sampler (right).} 
\label{fig:euler}
\end{figure}

We observe that, at the same number of function evaluations (NFEs), our GADD achieves a substantially smaller final error than both the Euler and $\theta$-Trapezoidal methods. This empirical advantage aligns with our theoretical result in \Cref{thm:conv-unif}: to attain a final total variation error of $\eps$, the required number of steps (and thus the NFE) scales as $\calO(\polylog (\eps^{-1}))$, which is significantly lower than that of existing samplers, whose complexity scales as $\calO(\mathrm{poly}(\eps^{-1}))$.

We further find that our GADD performs particularly well on spiky target distributions. In such settings, the vanilla Gibbs sampler is known to suffer from poor mixing behavior \cite{levin2017markov-chain-book}. In contrast, by leveraging the reverse diffusion process, GADD remains effective and achieves strong performance even under these challenging conditions. This observation provides empirical support for the \emph{warm-start} effect predicted by our theory, as discussed under \Cref{thm:conv-unif}.

\subsection{Experiment on Text Data}

We next consider sampling from text data. We first train a small SEDD Uniform model on the WikiText103 dataset using the suggested hyperparameters as in \cite{lou2024entropy}. We then compare the zero-shot sampling performance of GADD against several predictors, including the vanilla Euler method and the $\theta$-Trapezoidal method, as well as the CTMC corrector, which we analyze in \Cref{subsec:campbell_corr}. We follow the approach in \cite{zhao2025informed} and fix the corrector score at each outer-loop, and we do not re-evaluate the score model within the corrector loop.

\begin{table}[t]
\centering
\resizebox{\linewidth}{!}{
\begin{tabular}{ccccc}
\toprule
Method & NFE=32 & NFE=64 & NFE=128 & NFE=256 \\
\midrule
Vanilla Euler \cite{lou2024entropy}
& $356.0290 \pm 20.519$ (9s) 
& $285.1672 \pm 35.1706$ (19s) 
& $283.0323 \pm 24.3$ (39s) 
& $275.3841 \pm 41.248$ (79s) \\

$\theta$-Trapezoidal \cite{ren2025fast} 
& $325.9064 \pm 20.1654$ (11s) 
& $267.6957 \pm 30.701$ (23s) 
& $255.1958 \pm 6.6874$ (48s) 
& $265.9823 \pm 18.50979$ (97s) \\

CTMC corrector \cite{campbell2022discrete}
& $378.0972 \pm 19.0272$ (10s) 
& $272.653 \pm 10.6409$ (21s) 
& $227.6811 \pm 15.499$ (44s) 
& $219.8395 \pm 24.515$ (89s) \\

GADD (ours)
& $\mathbf{255.630 \pm 42.2442 }$ (6s) 
& $\mathbf{192.1529 \pm 35.0539}$ (13s) 
& $\mathbf{161.5748 \pm 26.1866}$ (26s) 
& $\mathbf{149.6407 \pm 17.4786}$ (53s) \\
\bottomrule
\end{tabular}
}
\smallskip
\caption{Zero-shot generative perplexity across different methods. Best result in \textbf{bold}.}
\label{tab:text}
\end{table}

In \Cref{tab:text}, we report the generative perplexity of the zero-shot generated outputs at fixed NFEs.\footnote{For example at NFE=32, the vanilla Euler method takes 32 predictor steps to complete, whereas GADD takes 16 predictor steps and 16 corrector loops.} Our proposed GADD method consistently achieves the best performance across all settings. In addition to this metric, we also report the actual wall-clock time. Notably, GADD exhibits lower wall-clock time compared to competing methods. This efficiency gain is likely due to differences in per-step computational cost: each Euler step updates all tokens simultaneously, whereas a random-scan Gibbs correction step updates only a single token per iteration. As a result, our method achieves a favorable trade-off between sample quality and computational efficiency.

The superiority of GADD over the CTMC corrector further supports our theoretical result in \Cref{thm:ctmc_corr}. While the complexity of the CTMC corrector scales as $\calO(\poly(\eps^{-1}))$, our GADD requires only $\calO(\polylog(\eps^{-1}))$, which is much smaller.

\subsection{Experiment on Monophonic Music Data}

\begin{table}[t]
\centering
\begin{tabular}{cccc}
\toprule
Method & Hellinger & Perplexity & Next-tokens Pred Err  \\
\midrule
Conditional SEDD \cite{lou2024entropy} 
& $0.1759 \pm 0.0061$
& $2.9255 \pm 0.0553$ 
& $0.3608 \pm 0.0141$ 
\\

+ GADD ($L=2$) 
& $0.0793 \pm 0.0044$
& $2.8295 \pm 0.0489$ 
& $0.3503 \pm 0.0223$ 
 \\

+ GADD ($L=5$)  
& $\mathbf{0.0780 \pm 0.0061}$ 
& $\mathbf{2.8108 \pm 0.0565}$ 
& $0.3373 \pm 0.0159$ 
\\

+ GADD ($L=50$) 
& $0.0817 \pm 0.0027$
& $2.8429 \pm 0.0568$ 
& $\mathbf{0.3238 \pm 0.0142}$ 
 \\
\bottomrule
\end{tabular}
\smallskip
\caption{Conditional generation qualities using Lakh pianoroll dataset \cite{dong2018lakh-dataset}. Best result in \textbf{bold}.}
\label{tab:music}
\vspace{-5mm}
\end{table}

Beyond text data, we also conduct experiments on zero-shot conditional discrete diffusion models \cite{lou2024entropy,liang2025conditional} on monophonic music data, obtained from \cite{dong2018lakh-dataset} and following the pre-processing steps in \cite{campbell2022discrete}. Different from unconditional sampling, the goal is to sample from the conditional distribution given part of the observable sequence, but in the zero-shot way without extra-training. In our setting, the goal is to complete a musical sequence given the first 100 notes. To evaluate the effect of correction, we compare: (i) the vanilla Conditional SEDD \cite{lou2024entropy}, a well-known zero-shot conditional discrete diffusion sampler, and (ii) the same model augmented with $L$ Gibbs correction steps (where $L \equiv L_k,~\forall k$).  We again fix the NFEs across all methods during evaluation.
Also, in addition to standard metrics such as the mean Hellinger distance of the histogram \cite{campbell2022discrete} and the generative perplexity (for which we trained an autoregressive model to evaluate), we introduce an additional metric: the mean prediction error over the next $F=4$ tokens, designed to capture the short-term temporal consistency in the generated sequences.

From \Cref{tab:music}, incorporating Gibbs correction steps yields significant improvements across all considered metrics. This highlights the generality of our Gibbs corrector, extending effectively even to conditional generation settings. We also observe that increasing the number of correction steps improves generation quality when $L$ is moderate. However, the gains diminish as $L$ grows, becoming marginal for larger values of $L$.In fact, the sample quality may saturate and even degrade as $L$ becomes large, due to stale correctors and accumulated random noise.

\section{Conclusion}

In this paper, we have developed an accelerated sampling algorithm for uniform-rate discrete diffusion models with provable convergence guarantees. Remarkably, our GADD algorithm achieved $\calO(\polylog(\eps^{-1}))$ convergence rate.
The accelerated rates were achieved by a novel Gibbs-based corrector that constructs posterior conditionals directly from the score function, thus without requiring additional training. We also demonstrated the superiority of GADD over previous CTMC correctors, along with its effectiveness on real-data experiments. An interesting direction for future work is to develop a theoretical understanding of predictor-corrector methods for masked diffusion models.

\section*{Acknowledgements}
    The work was supported in part by the U.S. National Science Foundation under the grants: NSF AI Institute (AI-EDGE) 2112471, ECCS-2413528, CNS-2312836, CNS-2223452, CNS-2225561, and was sponsored by the Army Research Laboratory under Cooperative Agreement Number W911NF-23-2-0225. The views and conclusions contained in this document are those of the authors and should not be interpreted as representing the official policies, either expressed or implied, of the Army Research Laboratory or the U.S. Government. The U.S. Government is authorized to reproduce and distribute reprints for Government purposes notwithstanding any copyright notation herein.

\bibliography{diffusion}

\newpage
\appendix
\begin{center}
    \huge \textbf{Appendix}
\end{center}


\allowdisplaybreaks
\startcontents[section]
{
\hypersetup{linkcolor=blue}
\printcontents[section]{l}{1}{\setcounter{tocdepth}{2}}
}

\crefalias{section}{appendix} 





\section{Additional Related Works}
\label{app:works-more}

We present the most relevant related works in \Cref{sec:works}; additional references are now discussed below.

\textbf{Theory of Acceleration for Continuous Diffusion Samplers.} A substantial body of theory has established accelerated convergence guarantees for continuous diffusion models, including both SDE and ODE solvers. Broadly, two main acceleration strategies have been explored. The first aims to more accurately approximate the continuous-time reverse process, for example by using higher-order numerical solvers such as Runge-Kutta methods or successive refinement (e.g., \cite{wu2024runge-kutta,liang2024discrete,li2025higher-order,li2024accl-prov}).
The second acceleration strategy augments the diffusion process with a corrector step that steers particles toward the target distribution. This predictor-corrector paradigm has been widely adopted in empirical studies of both continuous and discrete diffusion models \cite{campbell2022discrete,song2020sde}. From a theoretical perspective, \cite{huang2024rev-trans-kern} proposed the RTK-MALA and RTK-ULD algorithms, which incorporate Metropolis-Adjusted Langevin Algorithm (MALA) and Underdamped Langevin Dynamics (ULD) as correctors, respectively. Notably, RTK-MALA achieves a logarithmic convergence rate in $\eps^{-1}$.

\textbf{Theory on Gibbs Sampling.} Since the seminal work of \cite{geman1984gibbs}, Gibbs samplers have been shown to be effective across a wide range of applications, including sampling from graphical models, feasible sets of graph coloring problems \cite{levin2017markov-chain-book}, determinantal point processes \cite{dpp2019}, mixtures of densities \cite{yves2021spectral-gap}, and Bayesian hierarchical models \cite{ascolani2024gibbs-dimfree}. An orthogonal line of work studies the parallelization of Gibbs samplers under specific graph structures \cite{gonzalez2011parallel-gibbs}.
Notably, all of these works assume access to the exact posterior likelihood, whereas in diffusion-based settings the posterior must be estimated, introducing additional challenges that are not addressed by existing analyses.

\textbf{Empirical Studies on the Acceleration of Masked Diffusion Models.} Beyond acceleration for uniform-rate models, several empirical studies have focused on masked diffusion models. \cite{lezama2023discrete} constructed a discrete MCMC transition path for correction and learns a separate corrector kernel. More recently, \cite{zhao2025informed} proposed an informed correction strategy. Notably, these approaches are primarily empirical and do not provide theoretical convergence guarantees.

We now discuss \cite{zhao2025informed} in more detail. That work employs an informed Gibbs corrector specifically tailored to {\em masked (i.e., absorbing-rate)} diffusion models, whereas our work focuses on the {\em uniform-rate} setting. While both of their and our approaches utilize Gibbs-type correctors, uniform-rate diffusion models pose distinct algorithmic challenges. In particular, the score functions in uniform-rate models need to be estimated differently from those in masked diffusion, where the masked score (i.e., a time-scaled clean-data distribution) naturally facilitates Gibbs updates. In contrast, our method relies on posterior likelihood estimators specifically designed for uniform-rate models (e.g., \eqref{eq:est-second} and \eqref{eq:est-first}), enabling principled Gibbs corrections in this setting.



\section{Variants of GADD and Posterior Likelihood Estimate}
\label{app:est-example}

In this section, we provide some variants of both the GADD algorithm in \Cref{alg:unif} and the posterior estimator in \eqref{eq:est-second}. We first provide a systematic-scan variant of \Cref{alg:unif} as follows.

\begin{algorithm}
\caption{GADD with Systematic-scan Gibbs Correctors}
\setcounter{AlgoLine}{1}
\SetKwInput{Input}{Input}
\SetKwInput{Return}{Return}

\Input{initialization $x_{t_N} \sim p_\text{init}$, discretization points $\cbrc{t_k}_{k=0}^N$ (with $t_N = T$ and $t_0 = \delta$), estimated score $s$, correction steps $\cbrc{L_k}_{k=0}^N$}
\label{alg:unif-sys}
\For{$k = N-1$ \KwTo $0$}{
    $z_0 = $ the output of the Euler update given $x_{t_{k+1}}$ (using \eqref{eq:def_euler}) \;
    
    \For{$\ell = 1$ \KwTo $L_k$}{
    
        Construct posterior $\hat{q}_{t_k}^i(\cdot|z_{\ell-1}^{-i})$ for each $i \in [d]$ (e.g., using \eqref{eq:est-second} or \eqref{eq:est-first})\;
        
        Resample $z_\ell^i \sim \hat{q}_{t_k}^i$ for each $i \in [d]$;
    }
    
    $x_{t_{k}} = z_{L_k}$\;
}
\Return{$x_{t_{0}}$}
\end{algorithm}

Notably, since the score network outputs density-ratio scores for all $i \in [d]$ in a single forward pass, this modification does not incur any additional NFE per corrector step.

Next we provide variants of some posterior likelihood estimators. While \eqref{eq:est-second} is one possible estimator for the posterior likelihood from the trained score, such a construction is not unique. Notice that, for any $y^i \in [S]$,
\[ q_{t}^i(x^i|x^{-i}) = \frac{\frac{q_t(x^i|x^{-i})}{q_t(y^i| x^{-i})}}{\sum_{x^i \in [S]} \frac{q_t(x^i|x^{-i})}{q_t(y^i| x^{-i})}} = \frac{\frac{q_t(x)}{q_t(x^{-i} \oplus_i y^i)}}{\sum_{x^i \in [S]} \frac{q_t(x)}{q_t(x^{-i} \oplus_i y^i)}}. \]
Here the last equality again follows from Bayes' rule.
Thus, another estimator is, by fixing any $y^i \in [S]$,
\begin{equation*}
    \hat{q}_{t}^i(x^i|x^{-i}) = \frac{s_t(x,x^{-i} \oplus_i y^i)}{\sum_{x^i \in [S]} s_t(x,x^{-i} \oplus_i y^i)},
\end{equation*}
or, with averaging,
\begin{equation}
\label{eq:est-first}
    \hat{q}_{t}^i(x^i|x^{-i}) = \frac{1}{S} \sum_{y^i \in [S]} \frac{s_t(x,x^{-i} \oplus_i y^i)}{\sum_{x^i \in [S]} s_t(x,x^{-i} \oplus_i y^i)}.
\end{equation}
With a similar argument, another estimator is
\[ \hat{q}_{t}^i(x^i|x^{-i}) = \frac{ \sum_{y^i \in [S]} s_t(x,x^{-i} \oplus_i y^i)}{ \sum_{y^i \in [S]} \sum_{x^i \in [S]} s_t(x,x^{-i} \oplus_i y^i)}. \]

\section{Details of Numerical Experiments}
\label{app:details}

In this section, we provide more details on our synthetic and real-data experiments in \Cref{tab:text}.
For the real-data experiments, we use a Google H100 GPU during training and L4 GPU during sampling.

\subsection{Synthetic Experiment}

For the left figure of \Cref{fig:euler}, we consider a synthetic autoregressive-type data model, where the likelihood of each token depends on the preceding $h=2$ tokens. We choose the hyperparameters (e.g., $L_k$ in GADD) such that the number of function evaluations (NFE) is the same across all methods. The final errors are averaged over 10000 Monte Carlo runs.

For the right figure of \Cref{fig:euler}, we consider a mixed-point model in which the initial distribution is a sparse mixture of singletons, with support size on the order of $\calO(d)$. Again, the final errors are averaged over 10000 Monte Carlo runs.

\subsection{Text Experiment}
\label{app:details-text}

The sequence length is $d = 128$ and the dictionary size is $S = 50257$. We adopt the small SEDD Uniform model \cite{lou2024entropy} with $\sigma_{\min}=10^{-4}$ and $\sigma_{\max}=20$.

For optimization, we use Adam with a learning rate of $3 \times 10^{-3}$, $\beta_1=0.9$, and $\beta_2=0.999$, and apply a linear warm-up over the first 2500 steps. We train the model for a total of 111K steps. For evaluation, we use the typical GPT-2 model obtained from the HuggingFace library \cite{2019gpt2}.

For sampling, we use the Euler method as the backbone across all methods. All reported results are averaged over 10 independent runs. For the $\theta$-Trapezoidal method \cite{ren2025fast}, we fix $\theta = 0.5$. For the CTMC correctors \cite{campbell2022discrete}, we apply one correction step with a step size set to $1.5\times$ that of the Euler discretization, as suggested in their original paper. 

For GADD, we employ two techniques similar to those in \cite{zhao2025informed}. First, we use parallel Gibbs sampling, where the score function is not re-evaluated after each token update. Specifically, we randomly select one token to update at each step, without re-evaluating the score. This significantly reduces NFEs, while incurring only a minor performance degradation when $L_k$ is small. We choose $L_k = 40$ in our experiment.
Second, we update only tokens whose estimated probabilities fall below a fixed threshold, which is set to 0.1 in our experiments.

\subsection{Music Experiment}
\label{app:details-music}

The sequence length is $d = 256$ and the dictionary size is $S = 129$. We use the dataset in \cite{campbell2022discrete} with the pre-trained SEDD Uniform model from \cite{chu2025split-gibbs}. 

For sampling, we condition on the first 100 tokens of each sequence. We fix the number of function evaluations (NFE) to 16 across all methods. During sampling, the conditioned tokens are strictly enforced after each predictor and corrector step. For GADD, instead of randomly sampling from one token location, we use systematic update over all tokens. Again, we only perform correction when the current probability is below the threshold.

For evaluation, we use the test set that consist of 200 music pieces. As in \cite{campbell2022discrete}, the mean-histogram metric compares the Hellinger distance between the histograms from the true and generated samples. As an additional metric, we report the prediction error rate over the next $F=4$ tokens to assess short-term temporal consistency.

To compute sequence perplexity, we train a separate autoregressive (AR) model on the dataset in \cite{campbell2022discrete}. The model uses $n_{\text{head}}=8$, with 6 Transformer layers. The feed-forward dimension is set to 1024, with dropout rate 0.1. The model is trained using a learning rate of $3\times 10^{-4}$ for 400 epochs.






    
    


\section{Proof of \texorpdfstring{\Cref{thm:conv-unif}}{Theorem~1}}
\label{app:thm:conv-unif-proof}

Overall, the error in \Cref{alg:unif} arises from four distinct sources:
\begin{enumerate}
    \item the outer-loop (or grand) initialization error at the start of the algorithm;
    \item the inner-loop initialization error at the beginning of each Gibbs sampling phase; and
    \item the inner-loop error resulting from imperfect score estimation.
\end{enumerate}
We analyze each of these error components separately and then combine them to derive the overall convergence guarantee for \Cref{alg:unif}.


To begin, the following result gives us an upper bound on the grand (i.e., outer-loop) initialization error. 


From \cite[Proposition~2]{zhang2025conv-disc} and \cite[Theorem~C.1]{ren2025stoc-int}, the initialization error satisfies that
\[ \KL{\rev{q}_0}{p_0} \lesssim (d \log S) e^{-T}. \]
Thus, by Pinsker's inequality, we have
\begin{equation} \label{eq:grand-init-tv}
    \TV{p_{\text{init}}}{q_T} \lesssim \sqrt{\KL{p_{\text{init}}}{q_T} } \stackrel{(i)}{\lesssim} \sqrt{d \log S} \cdot e^{-T/2}.
\end{equation}

\subsection{Step 1: Bounding Inner-loop Initialization Error}

Our next result concerns the inner-loop initialization error. Specifically, we ask if the error at the end of the previous step is sufficiently small, how large the error will be at the start of the current step. Intuitively, this error remains small due to the smoothness of the diffusion process, which in turn yields a warm-up condition.

\begin{lemma} \label{lem:perstep-init}
Fix adjacent steps $s < t$ such that $t-s=o(1)$. 
Suppose that $\TV{\hat{p}_{T-t,L}}{q_t} \leq \eps$. Then, for the uniform-rate CTMC, under \Cref{ass:score-reg} and with the Euler update, we have
\[ \TV{p_{T-s,0}}{q_s} \lesssim (t-s) d (M + S \max\{1,s^{-1}\}) + \eps. \]
\end{lemma}
\begin{proof}
    See \Cref{app:lem:perstep-init-proof}.
\end{proof}

Note that this result readily extends to predictors beyond the Euler predictor, provided that $t-s = o(1)$ and the transition rates remain bounded.

\subsection{Step 2: Bounding Inner-loop Estimation Error}

Before we start to analyze the estimation error, we recall the triangle inequality for the TV distance:
\begin{equation} \label{eq:tv-triangle}
    \TV{q_{t}}{\hat{p}_{T-t,L}} \leq \TV{q_{t}}{p_{T-t,L}} + \TV{\hat{p}_{T-t,L}}{p_{T-t,L}}.
\end{equation}
Here the second term is regarding the estimation error, which we analyze below. In particular, \Cref{prop:est-err} considers the error propagation in general random-scan Gibbs samplers with target distribution $\pi$.

\begin{lemma} \label{prop:est-err}
    Recall the random-scan Gibbs sampler in \eqref{eq:def-gibbs-kernel}. 
    Given a mismatched Gibbs kernel $\hat{\pi}$ that satisfies
    \[ \max_{\ell=1,\dots,L} \E_{i \sim w_i} \E_{z \sim p_\ell} \TV{\hat{\pi}^i(\cdot|z^{-i})}{\pi^i(\cdot|z^{-i})} \leq \epsilon,\]
    we have
    \[ \TV{\hat{p}_L}{p_L} \leq L \epsilon. \]
\end{lemma}
\begin{proof}
    See \Cref{app:prop:est-err-proof}.
\end{proof}

Now, we specialize \Cref{prop:est-err} to the case of uniform-rate discrete diffusion models, with the estimator in \eqref{eq:est-second} and under \Cref{ass:score-ass}.

\begin{lemma} \label{lem:est-err-diffusion}
Fix $t \in [\delta,T]$ and the target $\pi = q_{t}$. Consider the estimator in \eqref{eq:est-second}. Consider the uniform-rate CTMC, and suppose that \Cref{ass:score-ass,ass:score-reg} hold, then $\TV{\hat{p}_{T-t,L}}{p_{T-t,L}} \lesssim L M \eps_{\text{est}}$.
\end{lemma}
\begin{proof}
    See \Cref{app:lem:est-err-diffusion-proof}.
\end{proof}

To show that our proof does not depend on the estimator in \eqref{eq:est-second}, we have the following result which uses instead the estimator in \eqref{eq:est-first}.

\begin{lemma} \label{lem:est-err-diffusion-other}
Under the same conditions as in \Cref{lem:est-err-diffusion}, but instead consider the estimator in \eqref{eq:est-first}. Then, $\TV{\hat{p}_{T-t,L}}{p_{T-t,L}} \lesssim L M^2 \eps_{\text{est}}$.
\end{lemma}
\begin{proof}
    See \Cref{app:lem:est-err-diffusion-proof-other}.
\end{proof}

Since the only difference is in the polynomial order of $M$, and the final dependence on $M$ is only logarithmic, the same convergence rate will hold for the estimator in \eqref{eq:est-first}. For this reason, we will only focus on the estimator in \eqref{eq:est-second} in the following.

\subsection{Step 3: Combining all parts together}

We are now ready to bring the pieces together. Suppose we choose $T \asymp \log(d (\log S)/\eps^2)$, by \eqref{eq:grand-init-tv}, the grand initialization error satisfies
\[ \TV{p_{\text{init}}}{q_T} \leq \eps. \]
Also, when we choose $\kappa \asymp \brc{\log \log (d/\eps^2)}^{-1}$, the total number of outer-loop steps satisfies that (cf. \cite{chen2023improved})
\[ N \asymp \frac{T + \log \delta^{-1}}{\kappa} \lesssim \brc{\log\log(d/\eps^2)} (\log(d/\eps^2) + \log \delta^{-1}). \]

Next we use an inductive argument by inducting on $k=N-1,\dots,0$.
Suppose, for purpose of induction, that we already have 
\[ \TV{\hat{p}_{T-t_{k+1},L_{k+1}}}{q_{t_{k+1}}} \lesssim L_{k+1} \eps. \]
Then, at step $k$, by \Cref{lem:perstep-init} and under \Cref{ass:score-reg}, we have
\[ \TV{p_{T-t_k,0}}{q_{t_k}} \lesssim \kappa d (M + S \max\{1,t_{k}^{-1}\}) + L_{k+1} \eps. \]

We now consider two sources of errors. The first is one that comes from the Gibbs sampling process. Following from standard analysis (cf. \eqref{eq:wass-conv-ass}),
\begin{align*}
    \TV{p_{T-t_k,L_k}}{q_{t_k}} &\leq \rmW_1 \brc{ p_{T-t_k,L_k}, q_{t_k} } \\
    &\leq (1-\rho_{t_k})^{L_k} \rmW_1 \brc{ p_{T-t_k,0}, q_{t_k} } \\
    &\leq (1-\rho_{t_k})^{L_k} d \cdot \TV{p_{T-t_k,0}}{q_{t_k}}.
\end{align*}
Thus, in order for $\TV{p_{T-t_k,L_k}}{q_{t_k}} \lesssim \eps$, it suffices to have
\begin{align*}
    L_k &\asymp \frac{\log\brc{\frac{d}{\eps} \cdot \TV{p_{T-t_k,0}}{q_{t_k}} }}{\rho_{t_k}} \\
    &\lesssim \frac{\log\brc{\frac{d}{\eps} \cdot (\kappa d (M + S \max\{1,t_k^{-1}\}) + L_{k+1} \eps) }}{\rho_{t_k}} \\
    &\lesssim \frac{\log\brc{\frac{d^2}{\eps} (M + S \delta^{-1}) }}{\rho_{t_k}}.
\end{align*}
The second error comes from score estimation. From \Cref{lem:est-err-diffusion} and under \Cref{ass:score-ass,ass:score-reg}, we have
\[ \TV{\hat{p}_{T-t_k,L_k}}{p_{T-t_k,L_k}} \lesssim L_k M \eps_{\text{est}}. \]
Thus, the total error at step $k$ is
\begin{align*}
    \TV{\hat{p}_{T-t_k,L_k}}{q_{t_k}} &\leq \TV{p_{T-t_k,L_k}}{q_{t_k}} + \TV{\hat{p}_{T-t_k,L_k}}{p_{T-t_k,L_k}} \\
    &\lesssim \eps + L_{k} M \eps_\text{est}.
\end{align*}
As long as $\eps_{\text{est}} = O(\eps/M)$, this verifies the inductive hypothesis.

Therefore, when $\eps_{\text{est}} = O \brc{ \frac{\rho_{t_0}}{M} \eps }$, we have $\TV{\hat{p}_{T-\delta,L_0}}{q_{\delta}} \lesssim \eps$. Meanwhile, the total number of steps satisfy (with $\delta \asymp \frac{\eps}{d}$):
\begin{align*}
    N_{total} &= N + \sum_{k=N-1}^0 L_{k} \\
    &\lesssim \sum_{k=N-1}^0 \frac{\log\brc{\frac{d^3}{\eps^2} S + \frac{d^2}{\eps} M }}{\rho_{t_k}} \\
    &\lesssim \log(d/\eps^2) \frac{\log\brc{\frac{d^3}{\eps^2} S + \frac{d^2}{\eps} M }}{\rho_*}.
\end{align*}
The proof is now complete.

\section{Proof of \texorpdfstring{\Cref{thm:ctmc_corr}}{Theorem~2}}
\label{app:thm:ctmc_corr-proof}

Fix the outer-loop step $t_k$. Different from Gibbs sampling, the corrector CTMC is a (time-homogeneous) continuous-time process indexed by a continuous variable, denoted as $\tau$ (note that this is different from the outer-loop index $t$). Also, let the terminal time for the corrector CTMC be $T_k$, and let the corrector step-size be $\eta_k$. (Here we do not need to shrink step-sizes towards the end because the CTMC rate is regular.) Note that $L_k = T_k / \eta_k$. We write $\Tilde{p}_{T-t_k,\tau}$ as the (ideal and clean) continuous-time sampling probability at continuous-time $\tau$.

As before, the quantity of interest is $\TV{\hat{p}_{T-t_k,L_k}}{q_{t_k}}$. We first decompose the total error as
\begin{equation} \label{eq:campbell-corr-main}
    \TV{\hat{p}_{T-t_k,L_k}}{q_{t_k}} \leq \TV{\Tilde{p}_{T-t_k,T_k}}{q_{t_k}} + \TV{\hat{p}_{T-t_k,L_k}}{\Tilde{p}_{T-t_k,T_k}}.
\end{equation}
Here the first term is the typical convergence error of a CTMC with time-homogeneous rate matrix. Classical results \cite[Corollary~2.8]{bobkov2006lsi} have established that, for any initial distribution $p_{T-t_k,0}$,
\[ \TV{\Tilde{p}_{T-t_k,T_k}}{q_{t_k}}^2 \leq 2 \log \frac{1}{\min_x q_{t_k}(x)} \cdot e^{-2 \mu_k T_k} \lesssim d (\log S) e^{-2 \mu_k T_k}, \]
where $\mu_k$ is the modified log-Sobolev constant of the rate matrix $R^c_{k}$. Since the reverse rate $\rev{R}_{t_k}$ ultimately depends on the underlying target distribution $q_{t_k}$, $\mu_k$ is also data-dependent, but it is independent of the target accuracy $\eps$.


We now consider the second term in \eqref{eq:campbell-corr-main}, which corresponds to the discretization error. The result is included in the following lemma. While we only consider the case without estimation error for simplicity, it can be easily adapted to the more general case with standard decomposition (e.g., \cite{liang2025sampler}).

\begin{lemma} \label{lem:campbell-discretization-err}
Suppose that there is no estimation error. Under \Cref{ass:score-reg} and using the $\tau$-leaping or the Euler method, we have
\[ \KL{\Tilde{p}_{T-t_k,T_k}}{p_{T-t_k,L_k}} \lesssim \eta_k T_k d^2 M^2 \log(M S).  \]
\end{lemma}

\begin{proof}
See \Cref{app:lem:campbell-discretization-err-proof}.
\end{proof}

Therefore, continuing from \eqref{eq:campbell-corr-main}, with perfect estimation, we have
\begin{equation*}
    \TV{p_{T-t_k,L_k}}{q_{t_k}} \lesssim \sqrt{d (\log S)} e^{-\mu_k T_k} + \sqrt{\eta_k T_k d^2 M^2 \log(M S)}.
\end{equation*}
In order to have $\TV{p_{T-t_k,L_k}}{q_{t_k}} \lesssim \eps$, it suffices to have
\[ T_k = O\brc{\mu_k^{-1} \log\frac{\sqrt{d (\log S)}}{\eps}},\quad \eta_k = \Tilde{\Theta}\brc{\mu_k \frac{\eps^2}{d^2 M^2} }, \]
which implies that $L_k = O\brc{\frac{d^2 M^2}{\mu_k^2 \eps^2}}$.
Then, with the same outer-loop step-sizes as in \Cref{thm:conv-unif} and following similar analyses, we have (ignoring log-dependencies)
\[ N_{total} = N + \sum_{k=N-1}^0 L_{k} \lesssim \frac{d^2 M^2}{\mu_*^2 \eps^2}. \]
Thus, even without estimation error, we have
\[ N_{total} = \poly(\eps^{-1}). \]

\section{Proofs of Auxiliary Lemmas}

We provide proofs for all helping lemmas in this section.

\subsection{Proof of \texorpdfstring{\Cref{lem:perstep-init}}{Lemma~1}}
\label{app:lem:perstep-init-proof}

Write $\kappa := t-s$. We first note that when $\kappa = o(1)$,
\[ q_{t|s}(x|z) = \ind{x=z} + R(z,x) \kappa + o(\kappa). \]

Thus, we have
\begin{align}
\label{eq:lem:perstep-init-main1}
    \frac{q_t(x)}{q_s(x)} &= \frac{\sum_{z} q_{t|s}(x|z) q_s(z)}{q_s(x)} \nonumber \\
    &= \frac{q_s(x) + \kappa R(x,x) q_s(x) + \kappa \sum_{z \neq x} R(z,x) q_s(z) + o(\kappa)}{q_s(x)} \nonumber \\
    &\stackrel{(i)}{=} \frac{q_s(x) + \kappa R(x,x) q_s(x) + \kappa \sum_{z: \Ham{z,x}=1} R(z,x) q_s(z) + o(\kappa)}{q_s(x)} \nonumber \\
    &= 1 + \kappa R(x,x) + \kappa \sum_{z: \Ham{z,x}=1} R(z,x) \frac{q_s(z)}{q_s(x)} + o(\kappa).
\end{align}
Here for $(i)$ we note that $R(z,x) = 0$ if $\Ham{z,x} \geq 2$. 


Another helpful relationship is regarding the (Euler) update from $\hat{p}_{T-t,L}$ to $p_{T-s,0}$, as follows. Indeed, such an update is coordinate-wise independent as $p_{T-s|T-t}(x_s|x_t) = \prod_{i=1}^d p^i_{T-s|T-t}(x_s^i|x_t^i)$. Further, from \eqref{eq:def_euler}, $p^i_{T-s|T-t}(x_s^i|x_t^i) = \hat{R}_{T-t}^i(x_t^i,x_s^i) (t-s)$ whenever $x_s^i \neq x_t^i$. Thus, under \Cref{ass:score-reg}, we have $p^i_{T-s|T-t}(x_s^i|x_t^i) \lesssim M S^{-1} (t-s)$. Also define the probability $\bm{\delta}^i(x_s^i|x_t^i) = 1$ only when $x_s^i = x_t^i$, and let $\bm{\delta}(x_s|x_t) = \prod_{i=1}^d \bm{\delta}^i(x_s^i|x_t^i)$. Then,
\begin{align}
\label{eq:lem:perstep-init-main2}
    \TV{p_t}{p_s}
    &\leq \E_{x_t\sim p_t}\sbrc{ \TV{\bm{\delta}(\cdot|x_t)}{p_{T-s|T-t}(\cdot|x_t)} } \nonumber\\
    &\leq \E_{x_t\sim p_t}\sbrc{ \sum_{i=1}^d \TV{\bm{\delta}^i(\cdot|x^i_t)}{p^i_{T-s|T-t}(\cdot|x^i_t)} } \nonumber \\
    &\lesssim d M (t-s).
\end{align}

With the results above, we have
\begin{align}
\label{eq:lem:perstep-init-main}
    \TV{p_{T-s,0}}{q_s} &\leq \TV{q_t}{q_s} + \TV{\hat{p}_{T-t,L}}{q_t} + \TV{\hat{p}_{T-t,L}}{p_{T-s,0}} \nonumber \\
    &= \frac{1}{2} \E_{x \sim q_t} \abs{ \frac{q_s(x)}{q_t(x)} - 1} + \TV{\hat{p}_{T-t,L}}{q_t} + \TV{\hat{p}_{T-t,L}}{p_{T-s,0}} \nonumber \\
    &\stackrel{(ii)}{\lesssim} \kappa \cdot \max_x  \brc{ R(x,x) + \sum_{z: \Ham{z,x}=1} R(z,x) \frac{q_s(z)}{q_s(x)} } + \eps + d M \kappa \nonumber\\
    &\stackrel{(iii)}{\lesssim} \kappa d (M + S \max\{1,s^{-1}\}) + \eps,
\end{align}
where $(ii)$ follows from \eqref{eq:lem:perstep-init-main1}, \eqref{eq:lem:perstep-init-main2} and the error assumption at time $t$, and $(iii)$ follows from the typical score upper-bound as in \cite[Lemma~2]{liang2025sampler}.
The proof is now complete.

\subsection{Proof of \texorpdfstring{\Cref{prop:est-err}}{Lemma~2}}
\label{app:prop:est-err-proof}

First, note that with random-scan Gibbs sampler, for $x$ and $y$ which only differ on the $i$-th index, we have
\[ p_{\ell+1|\ell}(y|x) = w_i \pi^i(y^i|x^{-i}),\quad \hat{p}_{\ell+1|\ell}(y|x) = w_i \hat{\pi}^i(y^i|x^{-i}). \]
Thus,
\begin{align*}
    \TV{\hat{p}_L}{p_L} &= \frac{1}{2} \sum_{z_L} \abs{\hat{p}_L(z_L) - p_L(z_L)} \\
    &= \frac{1}{2} \sum_{z_L} \abs{\sum_{z_0} p_0(z_0) \brc{\hat{p}_{L|0}(z_L|z_0) - p_{L|0}(z_L|z_0)} }\\
    &\leq \frac{1}{2} \sum_{z_L,z_0} p_0(z_0) \abs{\hat{p}_{L|0}(z_L|z_0) - p_{L|0}(z_L|z_0)},
\end{align*}
where the last line follows from Jensen's inequality. Here the inner difference can be upper-bounded as
\begin{align*}
    &\abs{\hat{p}_{L|0}(z_L|z_0) - p_{L|0}(z_L|z_0)} \\
    &\stackrel{(i)}{=} \abs{\sum_{z_1} \brc{ \hat{p}_{L|1}(z_L|z_1) \hat{p}_{1|0}(z_1|z_0) - p_{L|1}(z_L|z_1) p_{1|0}(z_1|z_0) } } \\
    &\leq \underbrace{\abs{\sum_{z_1} \hat{p}_{L|1}(z_L|z_1) \brc{\hat{p}_{1|0}(z_1|z_0) - p_{1|0}(z_1|z_0)} } }_{T_1} + \underbrace{\abs{\sum_{z_1} p_{1|0}(z_1|z_0) \brc{ \hat{p}_{L|1}(z_L|z_1)  - p_{L|1}(z_L|z_1)}  } }_{T_2},
\end{align*}
where $(i)$ follows because the Gibbs update step is Markov. Here note that for $T_2$, with Jensen's inequality,
\begin{align*}
    \frac{1}{2} \sum_{z_L,z_0} p_0(z_0) T_2 &= \frac{1}{2} \sum_{z_L,z_0} p_0(z_0) \abs{\sum_{z_1} p_{1|0}(z_1|z_0) \brc{ \hat{p}_{L|1}(z_L|z_1)  - p_{L|1}(z_L|z_1)}  } \\
    &\leq \frac{1}{2} \sum_{z_L,z_0,z_1} p_0(z_0) p_{1|0}(z_1|z_0) \abs{\hat{p}_{L|1}(z_L|z_1)  - p_{L|1}(z_L|z_1) } \\
    &= \frac{1}{2} \sum_{z_L,z_1} p_{1}(z_1) \abs{\hat{p}_{L|1}(z_L|z_1)  - p_{L|1}(z_L|z_1) }
\end{align*}
which establishes a recursive relationship. For the first term, we use the following helpful relationship: for arbitrary functions $q_1(x),q_2(x),a(x)$ such that $a(x) \geq 0$,
\begin{align*}
    &\abs{\sum_{x} \brc{ q_1(x) - q_2(x)} a(x)}\\
    &\leq \abs{\sum_{x:q_1 \geq q_2} \brc{ q_1(x) - q_2(x)} a(x)} + \abs{\sum_{x:q_1 \leq q_2} \brc{ q_2(x) - q_1(x)} a(x)}\\
    &= \sum_{x:q_1 \geq q_2} \brc{ q_1(x) - q_2(x)} a(x) + \sum_{x:q_1 \leq q_2} \brc{ q_2(x) - q_1(x)} a(x) \quad \text{(since $a(x) > 0$)}.
\end{align*}
Thus, since $\hat{p}_{L|1} \in (0,1]$,
\begin{align*}
    \frac{1}{2} \sum_{z_L,z_0} p_0(z_0) T_1 &= \frac{1}{2} \sum_{z_L,z_0} p_0(z_0) \abs{\sum_{z_1} \hat{p}_{L|1}(z_L|z_1) \brc{\hat{p}_{1|0}(z_1|z_0) - p_{1|0}(z_1|z_0)} } \\
    &\leq \frac{1}{2} \sum_{z_L,z_0} p_0(z_0) \bigg( \sum_{z_1:\hat{p}_{1|0} \geq p_{1|0}} \hat{p}_{L|1}(z_L|z_1) \brc{\hat{p}_{1|0}(z_1|z_0) - p_{1|0}(z_1|z_0)} \\
    &\qquad + \sum_{z_1:\hat{p}_{1|0} \leq p_{1|0}} \hat{p}_{L|1}(z_L|z_1) \brc{p_{1|0}(z_1|z_0) - \hat{p}_{1|0}(z_1|z_0)} \bigg)\\
    &= \frac{1}{2} \sum_{z_0} p_0(z_0) \bigg( \sum_{z_1:\hat{p}_{1|0} \geq p_{1|0}} \brc{\hat{p}_{1|0}(z_1|z_0) - p_{1|0}(z_1|z_0)} \\
    &\qquad + \sum_{z_1:\hat{p}_{1|0} \leq p_{1|0}} \brc{p_{1|0}(z_1|z_0) - \hat{p}_{1|0}(z_1|z_0)} \bigg) \\
    &= \frac{1}{2} \sum_{z_0,z_1} p_0(z_0) \abs{\hat{p}_{1|0}(z_1|z_0) - p_{1|0}(z_1|z_0)}.
\end{align*}
Combining two terms and providing upper-bounds recursively, we get
\begin{equation} \label{eq:prop:est-err-result}
    \begin{aligned}
        \TV{\hat{p}_L}{p_L} &\leq \sum_{\ell=0}^{L-1} \sum_{z_\ell} p_\ell(z_\ell) \brc{ \frac{1}{2} \sum_{z_{\ell+1}} \abs{\hat{p}_{\ell+1|\ell}(z_{\ell+1}|z_{\ell}) - p_{\ell+1|\ell}(z_{\ell+1}|z_{\ell})}  } \\
        &= \sum_{\ell=0}^{L-1} \sum_{z_\ell} p_\ell(z_\ell) \sum_{i=1}^d w_i \brc{ \frac{1}{2} \sum_{a} \abs{\hat{\pi}^i(a|z_{\ell}) - \pi^i(a|z_{\ell})}  } \\
        &\leq L \eps_\text{est}. 
    \end{aligned}
\end{equation}
The proof is now complete.

\subsection{Proof of \texorpdfstring{\Cref{lem:est-err-diffusion}}{Lemma~3}}
\label{app:lem:est-err-diffusion-proof}

We first fix $i \in [d]$ and $x^{-i}$. We have
\begin{align} \label{eq:lem:est-err-diffusion-proof-main}
    &\TV{\hat{q}_t^i(\cdot|x^{-i})}{q_t^i(\cdot|x^{-i})} \nonumber \\
    &= \frac{1}{2} \sum_{a \in [S]} \abs{\hat{q}_t^i(a|x^{-i}) - q_t^i(a|x^{-i})} \nonumber \\
    &= \frac{1}{2} \sum_{a \in [S]} \abs{\brc{ \sum_{y^i \in [S]} s_t(x^{-i} \oplus_i y^i,x^{-i} \oplus_i a) }^{-1} - q_t^i(a|x^{-i})} \nonumber \\
    &= \frac{1}{2} \sum_{a \in [S]} q_t^i(a|x^{-i}) \abs{\frac{1}{q_t^i(a|x^{-i}) \sum_{y^i \in [S]} s_t(x^{-i} \oplus_i y^i,x^{-i} \oplus_i a)} - 1} \nonumber \\
    &= \frac{1}{2} \sum_{a \in [S]} q_t^i(a|x^{-i}) \abs{\frac{1}{q_t^i(a|x^{-i})} - \sum_{y^i \in [S]} s_t(x^{-i} \oplus_i y^i,x^{-i} \oplus_i a)} \brc{\sum_{y^i \in [S]} s_t(x^{-i} \oplus_i y^i,x^{-i} \oplus_i a)}^{-1} \nonumber \\
    &\stackrel{(i)}{\leq} M S^{-1} \cdot \frac{1}{2} \sum_{a \in [S]} q_t^i(a|x^{-i}) \abs{\frac{1}{q_t^i(a|x^{-i})} - \sum_{y^i \in [S]} s_t(x^{-i} \oplus_i y^i,x^{-i} \oplus_i a)} \nonumber \\
    &\leq M S^{-1} \cdot \frac{1}{2} \sum_{a \in [S]} q_t^i(a|x^{-i})  \sum_{y^i \in [S]} \abs{\frac{q_t(x^{-i} \oplus_i y^i)}{q_t(x^{-i} \oplus_i a)} - s_t(x^{-i} \oplus_i y^i,x^{-i} \oplus_i a)},
\end{align}
where $(i)$ follows because, for each $a$,
\[ \sum_{y^i \in [S]} s_t(x^{-i} \oplus_i y^i,x^{-i} \oplus_i a) \geq M^{-1} S \implies \brc{\sum_{y^i \in [S]} s_t(x^{-i} \oplus_i y^i,x^{-i} \oplus_i a)}^{-1} \leq M S^{-1}.\]

With the above, we now specialize the condition of \Cref{prop:est-err} to the case where estimation error arises from inaccurate diffusion scores.
For each $i \in [d]$ and $\ell \in L$, we have
\begin{align*}
    &\E_{x \sim p_{T-t,\ell}} \TV{\hat{\pi}^i(\cdot|x^{-i})}{\pi^i(\cdot|x^{-i})} \\
    &\lesssim M S^{-1} \E_{x^{-i} \sim p^{-i}_{T-t,\ell}} \sum_{a \in [S]} q_t^i(a|x^{-i}) \sum_{y^i \in [S]} \abs{\frac{q_t(x^{-i} \oplus_i y^i)}{q_t(x^{-i} \oplus_i a)} - s_t(x^{-i} \oplus_i y^i,x^{-i} \oplus_i a)} \\
    &= M \cdot \E_{x^{-i} \sim p^{-i}_{T-t,\ell}} \E_{a \sim q_t(\cdot|x^{-i})} \sum_{y^i \in [S]} R_t(y,x) \abs{\frac{q_t(x^{-i} \oplus_i y^i)}{q_t(x^{-i} \oplus_i a)} - s_t(x^{-i} \oplus_i y^i,x^{-i} \oplus_i a) } \\
    &\leq M \eps_{\text{est}}
\end{align*}
where the last line follows from \Cref{ass:score-ass}.
The proof is now complete.



\subsection{Proof of \texorpdfstring{\Cref{lem:est-err-diffusion-other}}{Lemma~4}}
\label{app:lem:est-err-diffusion-proof-other}

The proof is very similar to \Cref{app:lem:est-err-diffusion-proof}, except with a slightly different upper bound on the per-step posterior mismatch. We use the estimator from \eqref{eq:est-first} such that
\[ \hat{q}_{t}^i(x^i|x^{-i}) = \frac{1}{S} \sum_{y^i \in [S]} \frac{s_t(x,x^{-i} \oplus_i y^i)}{\sum_{x^i \in [S]} s_t(x,x^{-i} \oplus_i y^i)}. \] 
Also note that
\[ q_{t}^i(x^i|x^{-i}) = \frac{1}{S} \sum_{y^i \in [S]} \frac{\frac{q_t(x)}{q_t(x^{-i} \oplus_i y^i)}}{\sum_{x^i \in [S]} \frac{q_t(x)}{q_t(x^{-i} \oplus_i y^i)}}. \] 
Again we fix $i \in [d]$ and $x^{-i}$. Then, we have
\begin{align*}
    &\TV{\hat{q}_t^i(\cdot|x^{-i})}{q_t^i(\cdot|x^{-i})} \\
    &= \frac{1}{2} \sum_{a \in [S]} \abs{\hat{q}_t^i(a|x^{-i}) - q_t^i(a|x^{-i})} \\
    &\leq \frac{1}{2} \sum_{x^i \in [S]} \frac{1}{S} \sum_{y^i \in [S]} \abs{ \frac{s_t(x,x^{-i} \oplus_i y^i)}{\sum_{x^i \in [S]} s_t(x,x^{-i} \oplus_i y^i)} - \frac{\frac{q_t(x)}{q_t(x^{-i} \oplus_i y^i)}}{\sum_{x^i \in [S]} \frac{q_t(x)}{q_t(x^{-i} \oplus_i y^i)}}} \\
    &\leq \frac{1}{2} \sum_{x^i \in [S]} \frac{1}{S} \sum_{y^i \in [S]} s_t(x, x^{-i} \oplus_i y^i) \abs{\frac{1}{\sum_{x^i \in [S]} s_t(x,x^{-i} \oplus_i y^i)} - \frac{1}{\sum_{x^i \in [S]} \frac{q_t(x)}{q_t(x^{-i} \oplus_i y^i)}}} \\
    &\qquad + \frac{1}{2} \sum_{x^i \in [S]} \frac{1}{S} \sum_{y^i \in [S]} \frac{1}{\sum_{x^i \in [S]} \frac{q_t(x)}{q_t(x^{-i} \oplus_i y^i)}} \abs{ s_t(x,x^{-i} \oplus_i y^i) - \frac{q_t(x)}{q_t(x^{-i} \oplus_i y^i)} } \\
    &\stackrel{(i)}{=} \frac{1}{2} \sum_{x^i \in [S]} \frac{1}{S} \sum_{y^i \in [S]} s_t(x, x^{-i} \oplus_i y^i) \abs{\frac{1}{\sum_{x^i \in [S]} s_t(x,x^{-i} \oplus_i y^i)} - q_t^i(y^i | x^{-i})} \\
    &\qquad + \frac{1}{2} \sum_{x^i \in [S]} \frac{1}{S} \sum_{y^i \in [S]} q_t^i(y^i | x^{-i}) \abs{ s_t(x,x^{-i} \oplus_i y^i) - \frac{q_t(x)}{q_t(x^{-i} \oplus_i y^i)} } \\
    &\stackrel{(ii)}{\leq} \frac{1}{2} M \sum_{a \in [S]} \abs{\brc{\sum_{y^i \in [S]} s_t(x^{-i} \oplus_i y^i,x^{-i} \oplus_i a)}^{-1} - q_t^i(a | x^{-i})} \\
    &\qquad + \frac{1}{2} \cdot \frac{1}{S} \sum_{a \in [S]} q_t^i(a | x^{-i}) \sum_{y^i \in [S]} \abs{ s_t(x^{-i} \oplus_i y^i,x^{-i} \oplus_i a) - \frac{q_t(x^{-i} \oplus_i y^i)}{q_t(x^{-i} \oplus_i a)} } \\
    &\stackrel{(iii)}{\leq} M^2 S^{-1} \cdot \frac{1}{2} \sum_{a \in [S]} q_t^i(a|x^{-i})  \sum_{y^i \in [S]} \abs{\frac{q_t(x^{-i} \oplus_i y^i)}{q_t(x^{-i} \oplus_i a)} - s_t(x^{-i} \oplus_i y^i,x^{-i} \oplus_i a)},
\end{align*}
where $(i)$ follows because
\[ \frac{1}{\sum_{x^i \in [S]} \frac{q_t(x)}{q_t(x^{-i} \oplus_i y^i)}} = \frac{1}{\sum_{x^i \in [S]} \frac{q_t^i(x^i|x^{-i})}{q_t^i(y^i | x^{-i})}} = q_t^i(y^i | x^{-i}), \]
$(ii)$ follows from \Cref{ass:score-reg} and by relabeling $(x^i,y^i)$ as $(y^i,a)$, and $(iii)$ follows from the proof of \Cref{lem:est-err-diffusion} (see \eqref{eq:lem:est-err-diffusion-proof-main}). Compared with \eqref{eq:lem:est-err-diffusion-proof-main}, the only difference is in the extra $M$ constant (that arises from normalization). The rest of the proof follows in the same manner as that of \Cref{lem:est-err-diffusion}, and the proof is complete.

\subsection{Proof of \texorpdfstring{\Cref{lem:campbell-discretization-err}}{Lemma~5}}
\label{app:lem:campbell-discretization-err-proof}

The proof follows similar to \cite[Theorem~2]{liang2025sampler}, but with several key differences. In particular, \cite[Theorem~2]{liang2025sampler} is specialized to the predictor error in the outer-loop CTMC. In such a case, the analyzed reverse CTMC has a corresponding forward process that provides nice geometric properties (e.g., dimension-wise factorization used in \cite[Lemma~2]{liang2025sampler}). This is different from our case with the corrector CTMC, which does not have such nice properties. On the other hand, different from the time-inhomogeneous reverse CTMC, the corrector CTMC has a time-homogeneous rate that helps the analysis.

We will omit all subscript $k$ whenever obvious. Define $\hat{R}^c_\tau$ as the rate corresponding to the discretized sampling process at time $\tau$. For $\tau \in [\eta_k \ell, \eta_k (\ell+1)]$, \textit{given the current state} $x_{\eta_k \ell}$, we have $\hat{R}_{\tau}^c(x_{\eta_k \ell},\cdot) = R_{\eta_k \ell}^c(x_{\eta_k \ell},\cdot)$. (Note that they are not equal for arbitrary state $x$ within the time duration, cf. \cite{liang2025sampler}.) Also note that $\Tilde{p}_{T-t_k,0} = p_{T-t_k,0}$.

Recall the definition of $R^c$ in \eqref{eq:campbell-corr-def}. Under \Cref{ass:score-reg}, a useful upper bound for $R^c$ is that, $\forall x \neq y$,
\begin{align} \label{eq:Rc-bound}
    R^c(x,y) &= R_{t_k}(x,y) + \hat{R}_{t_k}(x,y) = R_{t_k}(x,y) + R_{t_k}(y,x) s_{t_k}(y,x) \nonumber \\
    &\in \sbrc{S^{-1}(1+M^{-1}), S^{-1}(1+M)}.
\end{align}
Also note that $R^c(x,y) = 0$ whenever $\Ham{x,y} \geq 2$. Since this bound is uniform over $(x,y)$, from \cite[Appendix~B.5]{campbell2022discrete}, we also have the same uniform upper bound for $\hat{R}_{\tau}^c(x,y)$ as
\[ \hat{R}_{\tau}^c(x,y) \in \sbrc{S^{-1}(1+M^{-1}), S^{-1}(1+M)}. \]

Now, \cite[Theorem~1]{liang2025sampler} implies that
\begin{align*}
    &\KL{\Tilde{p}_{T-t_k,T_k}}{p_{T-t_k,L_k}} \\
    &\leq \sum_{\ell=0}^{L_k-1} \int_{\eta_k \ell}^{\eta_k (\ell+1)} \E_{x_\tau \sim \Tilde{p}_{T-t_k,\tau}} \underbrace{\sbrc{ \sum_{y:y \neq x_\tau} \hat{R}^c_{\tau}(x_\tau,y) - R^c(x_\tau,y) + R^c(x_\tau,y) \log \frac{R^c(x_\tau,y)}{\hat{R}_\tau^c(x_\tau,y)} } }_{=: g_\tau(x_\tau)}  \d \tau \\
    &\stackrel{(i)}{=} \sum_{\ell=0}^{L_k-1} \int_{\eta_k \ell}^{\eta_k (\ell+1)} \E_{\substack{x_\tau \sim \Tilde{p}_{T-t_k,\tau} \\ x_{\eta_k \ell} \sim \Tilde{p}_{T-t_k,\eta_k \ell}} } \sbrc{ g_\tau(x_\tau) - g_{\tau}(x_{\eta_k \ell}) }  \d \tau
\end{align*}
where $(i)$ follows because $g_{\tau}(x_{\eta_k \ell})=0$. Here
note that
\begin{align*}
    &\E \sbrc{ g_\tau(x_\tau) - g_{\tau}(x_{\eta_k \ell}) } \\
    &= \E_{x_{\eta_k \ell} \sim \Tilde{p}_{T-t_k,\eta_k \ell}} \sbrc{ \sum_{x_\tau} \Tilde{p}_{T-t_k,\tau|\eta_k \ell}(x_\tau | x_{\eta_k \ell}) g_\tau(x_\tau) - g_\tau(x_{\eta_k \ell}) } \\
    &\lesssim \E_{x_{\eta_k \ell} \sim \Tilde{p}_{T-t_k,\eta_k \ell}} \sbrc{ \sum_{x_\tau} \brc{\ind{x_\tau=x_{\eta_k \ell}} + R^c(x_{\eta_k \ell},x_\tau) (\tau-\eta_k \ell)} g_\tau(x_\tau) - g_\tau(x_{\eta_k \ell}) } \\
    &= (\tau-\eta_k \ell) \cdot \E_{x_{\eta_k \ell} \sim \Tilde{p}_{T-t_k,\eta_k \ell}} \sbrc{ \sum_{x_\tau: \Ham{x_\tau, x_{\eta_k \ell}}=1} R^c(x_{\eta_k \ell},x_\tau) g_\tau(x_\tau) } \\
    &\stackrel{(ii)}{\lesssim} (\tau-\eta_k \ell) \cdot (d S) \cdot \frac{M}{S} \cdot d M \log(S M),
\end{align*}
where in $(ii)$ we have used the uniform upper bound for $R^c$ in \eqref{eq:Rc-bound} and that for $g_\tau(x)$ as follows:
\begin{align*}
    \abs{g_\tau(x)} &= \abs{\sum_{y:y \neq x} \hat{R}_\tau^c(x,y) - R^c(x,y) + R^c(x,y) \log \frac{R^c(x,y)}{\hat{R}_\tau^c(x,y)}} \\
    &\leq \sum_{y:\Ham{x,y}=1} \abs{\hat{R}_\tau^c(x,y)} + \abs{R^c(x,y)} + \abs{R^c(x,y)} \cdot \abs{\log \frac{R^c(x,y)}{\hat{R}_\tau^c(x,y)}} \\
    &\lesssim d (S-1) \cdot \frac{M}{S} \log (S M).
\end{align*}
Thus,
\begin{align*}
    &\KL{\Tilde{p}_{T-t_k,T_k}}{p_{T-t_k,L_k}} \\
    &\lesssim d^2 S M \log(S M) \sum_{\ell=0}^{L_k-1} \int_{\eta_k \ell}^{\eta_k (\ell+1)} (\tau-\eta_k \ell) \d \tau \\
    &\lesssim d^2 S M \log(S M) \sum_{\ell=0}^{L_k-1} \eta_k^2 \\
    &= \eta_k T_k d^2 M^2 \log(S M).
\end{align*}
The proof is now complete.

\end{document}